\pdfoutput=1

\documentclass{article}

\usepackage{graphicx}
\usepackage{amsmath}
\usepackage{amssymb}
\usepackage{booktabs}
\usepackage{enumitem} %
\usepackage{booktabs} %
\usepackage{multirow} %
\usepackage{amssymb}%
\usepackage{pifont}%
\usepackage{transparent}
\usepackage{amsthm}   %

\usepackage{aliascnt}
\newtheorem{definition}{Definition}

\newaliascnt{lemma}{definition}
\newtheorem{lemma}[lemma]{Lemma}
\aliascntresetthe{lemma}

\newaliascnt{proposition}{definition}

\aliascntresetthe{proposition}

\newaliascnt{corollary}{definition}
\newtheorem{corollary}[corollary]{Corollary}
\aliascntresetthe{corollary}

\theoremstyle{definition}
\newaliascnt{remark}{definition}
\newtheorem{remark}[remark]{Remark}
\aliascntresetthe{remark}
\theoremstyle{plain}

\usepackage{hyperref}

\usepackage[preprint]{icml2026}
\hypersetup{pdftitle={LiBRA: Logit Fingerprints for Byzantine-Robust Client Selection in Federated Learning}, pdfsubject={}}

\renewcommand{\algorithmiccomment}[1]{\hfill$\triangleright$ #1}

\usepackage[capitalize]{cleveref}

\crefname{lemma}{Lemma}{Lemmas}
\crefname{proposition}{Proposition}{Propositions}
\crefname{corollary}{Corollary}{Corollaries}
\crefname{remark}{Remark}{Remarks}
\crefname{definition}{Definition}{Definitions}
\crefname{section}{Sec.}{Secs.}
\Crefname{section}{Section}{Sections}
\Crefname{table}{Table}{Tables}
\crefname{table}{Table}{Tables}

\def\algoname{LiBRA} %

\begin{document}

\icmltitlerunning{LiBRA: Logit Fingerprints for Byzantine-Robust Client Selection in Federated Learning}

\twocolumn[
  \icmltitle{\textsc{LiBRA}: Logit Fingerprints for Byzantine-Robust Client Selection in Federated Learning}

  \begin{icmlauthorlist}
    \icmlauthor{Saheed A. Tijani}{deakin}
    \icmlauthor{Frank Jiang}{deakin}
    \icmlauthor{Ming Liu}{deakin}
  \end{icmlauthorlist}
  \icmlaffiliation{deakin}{Centre for Cyber Security Research and Innovation (CSRI), Deakin University, Geelong, Australia}
  \icmlcorrespondingauthor{\mbox{Frank Jiang}}{frank.jiang@deakin.edu.au}
  \icmlcorrespondingauthor{\mbox{Saheed A. Tijani}}{stijani@deakin.edu.au}

  \icmlkeywords{Federated Learning, Byzantine Robustness, Client Selection}

  \vskip 0.3in
]

\printAffiliationsAndNotice{}

 \begin{abstract}
 Byzantine-robust aggregation shields federated learning from malicious updates, but the rules that filter out suspect updates do so by comparing clients' parameter vectors. That comparison requires access to every client's update, which secure aggregation exists to withhold, costs \(O(n^{2} d)\) in the parameter dimension \(d\), and is defined only between clients that share an architecture. We propose LiBRA, which compares clients on what their models compute rather than on the parameters that compute them, using a \emph{fingerprint} of \(C^{2}\) values for \(C\) classes, the class-conditional logit response of each client's model to pseudo-patterns optimised against it. We instantiate LiBRA as variants of Krum, Multi-Krum and Bulyan that change only what the dissimilarity score is computed on. On Fashion-MNIST and CIFAR-10 with three architectures and four backdoor and label-manipulation attacks under i.i.d.\ partitions, each variant holds every attack its parent holds, and on identical client updates under a backdoor attack the fingerprint ranks malicious clients within 0.006 AUC of the full parameter vector. The variants thus match the robustness of the rules they extend, while their trust decision reads no client parameter, compares up to 49{,}090 times fewer values, and can by construction compare clients whose architectures differ, at the cost of one pattern optimisation per client per round.

 \end{abstract}

\newcommand{\cmark}{\ding{51}}%
\newcommand{\xmark}{\ding{55}}%

\newcommand{\OverheadTable}{%
\begin{table*}[ht]
    \centering
    \small
    \caption{Per-round overhead of the logit substitution, measured from the runs this work reports. Timings are the median over the runs listed, all on one NVIDIA RTX A4000, and are never pooled across accelerator types. The same optimisation takes 18.8\,s on an A4000 and 10.0\,s on an A100, so a figure pooled over both describes no machine. Each client performs \emph{one} pseudo-pattern optimisation per round, in parallel with the other clients, so a round's critical path grows by one optimisation however many clients take part. Selection cost is stated in elements compared per client pair, which is the trust decision alone. The aggregation that follows is unchanged and is not counted.}
    \label{tab:overhead}
    \begin{tabular}{lccc}
    \toprule
    \textbf{Component} & \textbf{Fashion-MNIST} & \textbf{CIFAR-10 (CNN)} & \textbf{CIFAR-10 (ResNet-9)} \\
    \midrule
    \multicolumn{4}{l}{\emph{Client-side, per client per round (RTX A4000)}} \\
    Local training (10 epochs) & 1.25s & 4.96s & 24.11s \\
    Pseudo-pattern optimisation (1000 iterations) & 2.71s & 6.83s & 18.77s \\
    \quad as a share of local training & 217\% & 138\% & 78\% \\
    Runs behind the medians & 5 & 4 & 27 \\
    \midrule
    \multicolumn{4}{l}{\emph{Communication, per client per round}} \\
    Model parameters & 0.12 MB & 2.21 MB & 19.64 MB \\
    Fingerprint ($C^{2}=100$ floats) & 0.0004 MB & 0.0004 MB & 0.0004 MB \\
    \quad as a share of the parameter payload & 0.34\% & 0.018\% & 0.002\% \\
    \midrule
    \multicolumn{4}{l}{\emph{Server-side selection, elements compared per client pair}} \\
    Krum (parameters), $d$ & $29{,}130$ & $552{,}362$ & $4{,}909{,}002$ \\
    {\algoname} (fingerprints), $C^{2}$ & $100$ & $100$ & $100$ \\
    \quad reduction factor & $291\times$ & $5{,}524\times$ & $49{,}090\times$ \\
    \bottomrule
    \end{tabular}
\end{table*}
}

\newcommand{\krumVsKrumLite}{%
\begin{table*}[ht] %
    \small
    \setlength{\tabcolsep}{4pt}
    \caption{{\algoname} against the rules it adapts, by construction rather than measurement. A \checkmark$^{*}$ marks resilience inherited from the parent directly when the fingerprint selects, in every round, the clients the parameter distance would, and \Cref{lem:order} gives the separation under which fingerprint selection admits no Byzantine client. \emph{Parameter-free selection} decides which clients to admit without reading a client parameter. \emph{Architecture-agnostic selection} needs only a shared label space and input shape, since a fingerprint is $C \times C$ for every model, although no mixed-architecture cohort is evaluated.}
    \label{tab:krum-vs-krum-lite}
    \begin{tabular}{lcccc}
    \toprule
    \textbf{Algorithm} & \begin{tabular}{@{}c@{}}\textbf{Byzantine}\\\textbf{resilient}\end{tabular} & \begin{tabular}{@{}c@{}}\textbf{Parameter-free}\\\textbf{selection}\end{tabular} & \begin{tabular}{@{}c@{}}\textbf{Architecture-}\\\textbf{agnostic selection}\end{tabular} & \begin{tabular}{@{}c@{}}\textbf{Selection}\\\textbf{complexity}\end{tabular} \\
    \midrule
    (Multi-)Krum & \checkmark & \xmark & \xmark & \( O(n^2 d + m\, n^2 \log n) \) \\
    Bulyan       & \checkmark & \xmark & \xmark & \( O(n^2 d + \theta\, n^2 \log n) \) \\
    \textbf{(Multi-)Krum-Lite} & \checkmark$^{*}$ & \checkmark & \checkmark & \( O(n^2 C^{2} + m\, n^2 \log n) \) \\
    \textbf{Bulyan-Lite} & \checkmark$^{*}$ & \checkmark & \checkmark & \( O(n^2 C^{2} + \theta\, n^2 \log n) \) \\
    \bottomrule
    \end{tabular}
\end{table*}
}

\newcommand{\SelectionWallClock}{%
\begin{table*}[ht]
    \centering
    \small
    \setlength{\tabcolsep}{4pt}
    \caption{Selection time for one round of 10 clients, 90 pairwise comparisons, on an NVIDIA RTX A4000. Each figure is the median of three benchmark runs of 100 repetitions each. {\algoname} compares the same 100-value fingerprint for all three architectures, so it has one row. The last two columns are the point of the table. The element count falls by two to five orders of magnitude while the clock falls by less than one, because at $n=10$ a comparison is dominated by a fixed per-call cost rather than by the elements it touches.}
    \label{tab:wallclock}
    \begin{tabular}{lrrrr}
    \toprule
    \textbf{Rule} & \textbf{Elements} & \textbf{Time} & \textbf{Element} & \textbf{Measured} \\
                  & \textbf{per pair} & \textbf{(ms)} & \textbf{ratio}   & \textbf{speedup} \\
    \midrule
    Krum, Fashion-MNIST CNN & $29{,}130$    & 57.7  & ---              & --- \\
    Krum, CIFAR-10 CNN      & $552{,}362$   & 155.5 & ---              & --- \\
    Krum, ResNet-9          & $4{,}909{,}002$ & 242.3 & ---            & --- \\
    \midrule
    {\algoname}, any of the three & $100$   & 29.7  & ---              & --- \\
    \midrule
    \quad against Fashion-MNIST CNN & & & $291\times$      & $1.9\times$ \\
    \quad against CIFAR-10 CNN      & & & $5{,}524\times$  & $5.2\times$ \\
    \quad against ResNet-9          & & & $49{,}090\times$ & $8.2\times$ \\
    \bottomrule
    \end{tabular}
\end{table*}
}

\providecommand{\ExperimentalResults}{}%
\renewcommand{\ExperimentalResults}{%
\begin{table*}[ht!]
    \centering
    \small
    \setlength{\tabcolsep}{3pt}
    \caption{Clean accuracy and attack success (\%) for every rule and configuration under i.i.d.\ partitions. Each value is the median over three repetitions of the mean over the final 10 rounds, $\pm$ the median of the within-run standard deviations. Attack success is backdoor accuracy under a pixel trigger and malicious-task accuracy under a label attack. Bold rows are the {\algoname} variants. \dag\ marks repeats that disagree by more than 5 points. \ddag\ marks a configuration where the attack stays below 40\% even with no defence.}
    \label{tab:ExperimentalResults}
    \begin{tabular}{lcccccc}
    \toprule
    \multirow{2}{*}{\textbf{Aggregator}}
      & \multicolumn{2}{c}{\textbf{Fashion-MNIST}}
      & \multicolumn{2}{c}{\textbf{CIFAR-10 (CNN)}}
      & \multicolumn{2}{c}{\textbf{CIFAR-10 (ResNet-9)}} \\
    \cmidrule(lr){2-3} \cmidrule(lr){4-5} \cmidrule(lr){6-7}
    & Clean accuracy & Attack
    & Clean accuracy & Attack
    & Clean accuracy & Attack \\
    \midrule

    &  &  & \textbf{1-pixel attacks} &  &  & \\
    \midrule
    FedAvg                   &
                             89.87$\pm$0.12 & 93.95$\pm$0.25 &
                             82.02$\pm$0.19 & 98.35$\pm$0.19 &
                             82.84$\pm$0.21 & 99.91$\pm$0.02 \\
    Krum                     &
                             87.58$\pm$0.29 & 9.92$\pm$0.13 &
                             71.41$\pm$0.73 & 10.57$\pm$0.63 &
                             70.69$\pm$0.86 & 10.87$\pm$0.62 \\
    \textbf{Krum-Lite}       &
                             \textbf{88.08$\pm$0.33} & \textbf{9.88$\pm$0.20} &
                             \textbf{76.85$\pm$0.71} & \textbf{9.92$\pm$0.56} &
                             \textbf{74.91$\pm$1.15} & \textbf{10.26$\pm$0.64} \\
    Multi-Krum               &
                             89.02$\pm$0.12 & 10.07$\pm$0.09 &
                             77.37$\pm$0.32 & 10.37$\pm$0.31 &
                             77.21$\pm$0.42 & 10.66$\pm$0.43 \\
    \textbf{Multi-Krum-Lite} &
                             \textbf{89.28$\pm$0.17} & \textbf{9.91$\pm$0.07} &
                             \textbf{81.10$\pm$0.34} & \textbf{10.29$\pm$0.28} &
                             \textbf{81.10$\pm$0.31} & \textbf{10.25$\pm$0.47} \\
    Bulyan                   &
                             89.38$\pm$0.11 & 10.09$\pm$0.08 &
                             71.88$\pm$0.30\dag & 52.79$\pm$4.89\dag &
                             77.69$\pm$0.27 & 10.67$\pm$0.30 \\
    \textbf{Bulyan-Lite}     &
                             \textbf{88.81$\pm$0.20} & \textbf{9.87$\pm$0.09} &
                             \textbf{74.88$\pm$0.65} & \textbf{10.18$\pm$0.74} &
                             \textbf{76.55$\pm$1.29} & \textbf{10.54$\pm$0.79} \\
    \midrule

    &  &  & \textbf{9-pixel attacks} &  &  & \\
    \midrule
    FedAvg                   &
                             89.97$\pm$0.09 & 99.17$\pm$0.12 &
                             82.16$\pm$0.18 & 99.67$\pm$0.05 &
                             82.45$\pm$0.23 & 99.96$\pm$0.01 \\
    Krum                     &
                             87.61$\pm$0.22 & 9.86$\pm$0.26 &
                             70.99$\pm$1.07 & 10.03$\pm$0.85 &
                             70.73$\pm$1.16 & 10.78$\pm$0.56 \\
    \textbf{Krum-Lite}       &
                             \textbf{88.25$\pm$0.35} & \textbf{9.67$\pm$0.19} &
                             \textbf{76.16$\pm$1.00} & \textbf{8.71$\pm$0.89} &
                             \textbf{74.15$\pm$0.86} & \textbf{9.07$\pm$0.75} \\
    Multi-Krum               &
                             88.88$\pm$0.20 & 9.68$\pm$0.19 &
                             77.77$\pm$0.32 & 10.11$\pm$0.35 &
                             77.56$\pm$0.28 & 10.70$\pm$0.43 \\
    \textbf{Multi-Krum-Lite} &
                             \textbf{89.25$\pm$0.12} & \textbf{9.47$\pm$0.12} &
                             \textbf{81.28$\pm$0.32} & \textbf{8.48$\pm$0.38} &
                             \textbf{81.58$\pm$0.46} & \textbf{9.45$\pm$0.29} \\
    Bulyan                   &
                             89.05$\pm$0.13 & 10.49$\pm$0.13 &
                             72.46$\pm$0.28\dag & 13.56$\pm$0.98\dag &
                             77.26$\pm$0.50 & 10.75$\pm$0.42 \\
    \textbf{Bulyan-Lite}     &
                             \textbf{88.86$\pm$0.21} & \textbf{9.93$\pm$0.09} &
                             \textbf{74.48$\pm$0.44\dag} & \textbf{10.04$\pm$0.80\dag} &
                             \textbf{77.69$\pm$1.21} & \textbf{9.44$\pm$0.62} \\
    \midrule

    &  &  & \textbf{LabelFlip attacks} &  &  & \\
    \midrule
    FedAvg                   &
                             67.88$\pm$11.00 & 62.54$\pm$26.68 &
                             60.13$\pm$1.43 & 64.94$\pm$2.70 &
                             60.92$\pm$1.02 & 63.55$\pm$2.55 \\
    Krum                     &
                             87.67$\pm$0.30 & 6.02$\pm$0.75 &
                             71.86$\pm$0.71 & 11.04$\pm$1.39 &
                             69.82$\pm$0.90 & 12.73$\pm$1.34 \\
    \textbf{Krum-Lite}       &
                             \textbf{88.18$\pm$0.39} & \textbf{5.76$\pm$0.70} &
                             \textbf{76.27$\pm$0.67} & \textbf{8.24$\pm$1.08} &
                             \textbf{74.26$\pm$1.05} & \textbf{8.49$\pm$1.23} \\
    Multi-Krum               &
                             89.03$\pm$0.15 & 5.35$\pm$0.42 &
                             76.67$\pm$0.37 & 9.93$\pm$0.64 &
                             77.49$\pm$0.47 & 11.27$\pm$0.70 \\
    \textbf{Multi-Krum-Lite} &
                             \textbf{89.23$\pm$0.13} & \textbf{5.08$\pm$0.38} &
                             \textbf{81.59$\pm$0.23} & \textbf{5.89$\pm$0.51} &
                             \textbf{81.04$\pm$0.41} & \textbf{5.62$\pm$0.70} \\
    Bulyan                   &
                             89.13$\pm$0.11 & 5.18$\pm$0.33 &
                             72.76$\pm$0.31 & 12.78$\pm$1.37 &
                             77.09$\pm$0.35 & 11.29$\pm$0.73 \\
    \textbf{Bulyan-Lite}     &
                             \textbf{88.63$\pm$0.30} & \textbf{4.89$\pm$0.62} &
                             \textbf{74.19$\pm$0.44} & \textbf{8.33$\pm$1.05} &
                             \textbf{76.73$\pm$0.88} & \textbf{7.75$\pm$1.01} \\
    \midrule

    &  &  & \textbf{Targeted attacks} &  &  & \\
    \midrule
    FedAvg                   &
                             84.21$\pm$0.64 & 74.98$\pm$6.89 &
                             76.82$\pm$0.49 & 46.64$\pm$6.12 &
                             77.32$\pm$0.58 & 29.07$\pm$2.12\ddag \\
    Krum                     &
                             87.58$\pm$0.31 & 10.07$\pm$2.43 &
                             71.08$\pm$0.58 & 13.42$\pm$2.41 &
                             70.44$\pm$0.97 & 13.65$\pm$2.51 \\
    \textbf{Krum-Lite}       &
                             \textbf{88.31$\pm$0.36} & \textbf{10.54$\pm$2.85} &
                             \textbf{75.97$\pm$0.78} & \textbf{14.18$\pm$4.08} &
                             \textbf{74.68$\pm$1.32} & \textbf{15.21$\pm$6.40} \\
    Multi-Krum               &
                             88.80$\pm$0.11 & 10.09$\pm$1.09 &
                             77.54$\pm$0.23 & 16.31$\pm$1.73 &
                             77.98$\pm$0.31 & 12.59$\pm$1.55 \\
    \textbf{Multi-Krum-Lite} &
                             \textbf{89.26$\pm$0.18} & \textbf{9.13$\pm$1.40} &
                             \textbf{81.09$\pm$0.23} & \textbf{12.00$\pm$1.92} &
                             \textbf{80.65$\pm$0.44} & \textbf{10.37$\pm$2.14} \\
    Bulyan                   &
                             89.28$\pm$0.11 & 10.48$\pm$1.28 &
                             72.30$\pm$0.39 & 14.30$\pm$1.47 &
                             78.12$\pm$0.40 & 15.45$\pm$1.82 \\
    \textbf{Bulyan-Lite}     &
                             \textbf{88.93$\pm$0.25} & \textbf{6.23$\pm$1.15} &
                             \textbf{74.67$\pm$0.68} & \textbf{17.41$\pm$3.62} &
                             \textbf{77.47$\pm$1.00} & \textbf{19.98$\pm$3.49} \\
    \bottomrule
    \end{tabular}
\end{table*}
}

\providecommand{\NoFilterBaselines}{}%
\renewcommand{\NoFilterBaselines}{%
\begin{table*}[ht!]
    \centering
    \small
    \setlength{\tabcolsep}{4pt}
    \caption{The coordinate-wise median and the trimmed mean, the two rules that combine every update without selecting, under i.i.d.\ partitions. Each value is the median over three repetitions of the mean over the final 10 rounds, $\pm$ the median of the within-run standard deviations. Neither resists a pixel trigger in any configuration, so robustness in the benchmark comes from selecting clients.}
    \label{tab:NoFilterBaselines}
    \begin{tabular}{lcccccc}
    \toprule
    \multirow{2}{*}{\textbf{Rule}}
      & \multicolumn{2}{c}{\textbf{Fashion-MNIST}}
      & \multicolumn{2}{c}{\textbf{CIFAR-10 (CNN)}}
      & \multicolumn{2}{c}{\textbf{CIFAR-10 (ResNet-9)}} \\
    \cmidrule(lr){2-3} \cmidrule(lr){4-5} \cmidrule(lr){6-7}
    & Clean accuracy & Attack & Clean accuracy & Attack & Clean accuracy & Attack \\
    \midrule

    &  &  & \textbf{1-pixel attacks} &  &  & \\
    \midrule
    Median         & 89.57$\pm$0.12 & 85.79$\pm$1.58 & 78.15$\pm$0.22 & 89.95$\pm$0.78 & 81.78$\pm$0.24 & 95.59$\pm$0.56 \\
    Trimmed mean   & 89.51$\pm$0.11 & 87.82$\pm$1.25 & 81.94$\pm$0.23 & 98.66$\pm$0.16 & 82.86$\pm$0.24 & 99.84$\pm$0.04 \\
    \midrule

    &  &  & \textbf{9-pixel attacks} &  &  & \\
    \midrule
    Median         & 89.64$\pm$0.10 & 98.76$\pm$0.31 & 78.18$\pm$0.28 & 97.26$\pm$0.32 & 81.54$\pm$0.19 & 98.57$\pm$0.22 \\
    Trimmed mean   & 89.95$\pm$0.07 & 98.23$\pm$0.25 & 82.25$\pm$0.19 & 99.62$\pm$0.07 & 82.68$\pm$0.21 & 99.97$\pm$0.02 \\
    \midrule

    &  &  & \textbf{LabelFlip attacks} &  &  & \\
    \midrule
    Median         & 88.56$\pm$0.22 & 10.58$\pm$0.50 & 75.34$\pm$0.30 & 17.30$\pm$0.75 & 79.54$\pm$0.31 & 20.20$\pm$0.79 \\
    Trimmed mean   & 87.55$\pm$0.33 & 14.19$\pm$0.88 & 78.53$\pm$0.50 & 20.70$\pm$0.96 & 80.32$\pm$0.44 & 20.50$\pm$1.14 \\
    \midrule

    &  &  & \textbf{Targeted attacks} &  &  & \\
    \midrule
    Median         & 89.49$\pm$0.16 & 21.36$\pm$1.47 & 76.81$\pm$0.35 & 21.57$\pm$2.15 & 80.50$\pm$0.40 & 17.60$\pm$1.18 \\
    Trimmed mean   & 88.95$\pm$0.23 & 25.36$\pm$2.33 & 81.52$\pm$0.25 & 25.88$\pm$1.81 & 82.30$\pm$0.31 & 17.04$\pm$1.21 \\
    \bottomrule
    \end{tabular}
\end{table*}
}

\providecommand{\ResultsDelta}{}%
\renewcommand{\ResultsDelta}{%
\begin{table}[ht!]
    \centering
    \footnotesize
    \setlength{\tabcolsep}{3pt}
    \caption{Paired differences between each {\algoname} variant and the rule it replaces. Each entry is a median, with the number of configurations in which the variant leads as a subscript. Higher is better for clean accuracy, lower for attack success.}
    \label{tab:ResultsDelta}
    \begin{tabular}{lccc}
    \toprule
    \textbf{Variant} & \textbf{Cfgs} & \textbf{Clean} & \textbf{Attack success} \\
    \midrule
    Krum-Lite       & 11 & $+4.22_{11}$ & $-0.61_{9}$ \\
    Multi-Krum-Lite & 11 & $+3.55_{11}$ & $-0.96_{11}$ \\
    Bulyan-Lite     & 11 & $-0.19_{5}$ & $-1.31_{10}$ \\
    \bottomrule
    \end{tabular}
\end{table}
}

\providecommand{\AdmissionTable}{}%
\renewcommand{\AdmissionTable}{%
\begin{table}[ht!]
    \centering
    \footnotesize
    \setlength{\tabcolsep}{2pt}
    \caption{How often each rule admitted an adversary, by attack. An entry is the number of rounds in which at least one of the two adversaries present was admitted, of the rounds compared, over the final 10 rounds of every run. Lower is better and zero is perfect. Bolded rows are the variants of \Cref{sec:Applications}. Federated averaging, the coordinate-wise median and the trimmed mean are omitted, since they combine every update and exclude no client.}
    \label{tab:Admission}
    \begin{tabular}{lcccc}
    \toprule
    \textbf{Aggregator} & \textbf{1-pixel} & \textbf{9-pixel} & \textbf{LabelFlip} & \textbf{Targeted} \\
    \midrule
    Krum                      & 0/90 & 0/90 & 0/90 & 0/90 \\
    \textbf{Krum-Lite}        & 0/90 & 0/90 & 0/90 & 0/90 \\
    Multi-Krum                & 0/90 & 0/90 & 0/90 & 0/90 \\
    \textbf{Multi-Krum-Lite}  & 0/90 & 0/90 & 0/90 & 0/90 \\
    Bulyan                    & 24/90 & 5/90 & 0/90 & 0/90 \\
    \textbf{Bulyan-Lite}      & 0/90 & 4/90 & 0/90 & 1/90 \\
    \bottomrule
    \end{tabular}
\end{table}
}

\providecommand{\AblationAUC}{}
\providecommand{\AblationPrecision}{}
\providecommand{\AblationDelta}{}
\providecommand{\AblationDeltaInner}{}
\providecommand{\AblationSpread}{}
\providecommand{\PENDING}[1]{\textbf{[PENDING: #1]}}
\providecommand{\DEGEN}{\textit{n/a}}

\renewcommand{\AblationDelta}{%
\begin{table}[t]
\centering\small
\setlength{\tabcolsep}{4pt}
\caption{Change in detection AUC from the full parameter vector to the logit fingerprint, positive favouring the fingerprint. Median over 3 repeats of each configuration. \dag\ marks a range across repeats above 0.05.}
\label{tab:ablation-delta-p3}
\begin{tabular}{lccc}
\toprule
Selection & Fashion-MNIST & CIFAR-10 & ResNet-9 \\
\midrule
admits 1 of 10 & +0.000 & -0.006 & -0.003 \\
admits 3 of 10 & +0.000 & +0.006 & +0.000 \\
\bottomrule
\end{tabular}
\end{table}
}

\renewcommand{\AblationDeltaInner}{%
\centering\small
\setlength{\tabcolsep}{4pt}
\captionof{table}{Change in detection AUC from the full parameter vector to the logit fingerprint, positive favouring the fingerprint. Median over 3 repeats of each configuration. \dag\ marks a range across repeats above 0.05.}
\label{tab:ablation-delta-p3}
\begin{tabular}{lccc}
\toprule
Selection & Fashion-MNIST & CIFAR-10 & ResNet-9 \\
\midrule
admits 1 of 10 & +0.000 & -0.006 & -0.003 \\
admits 3 of 10 & +0.000 & +0.006 & +0.000 \\
\bottomrule
\end{tabular}
}

\renewcommand{\AblationAUC}{%
\begin{table*}[t]
\centering\small
\setlength{\tabcolsep}{4pt}
\caption{Detection AUC of each representation, scored on identical client updates in identical rounds. Only the logit representation drove selection. Median over 3 repeats of each configuration. Chance is 0.5. \dag\ marks a range across repeats above 0.05.}
\label{tab:ablation-auc-p3}
\begin{tabular}{lccccc}
\toprule
model & full params & rand.\ proj & penult. & logits & softmax \\
\midrule
\multicolumn{6}{l}{\emph{Selection:} admits 1 of 10 clients} \\
Fashion-MNIST & 1.000 & 0.982 & 0.998 & 1.000 & 0.702\dag \\
CIFAR-10 & 0.994 & 0.703\dag & 0.739 & 0.988 & 0.898 \\
ResNet-9 & 1.000 & 0.936\dag & 0.904\dag & 0.997 & 0.653\dag \\
\midrule
\multicolumn{6}{l}{\emph{Selection:} admits 3 of 10 clients} \\
Fashion-MNIST & 1.000 & 0.976 & 1.000 & 1.000 & 0.953 \\
CIFAR-10 & 0.993 & 0.724\dag & 0.752\dag & 0.999 & 0.895 \\
ResNet-9 & 1.000 & 0.983 & 0.876\dag & 1.000 & 0.551\dag \\
\bottomrule
\end{tabular}
\end{table*}
}

\renewcommand{\AblationPrecision}{%
\begin{table*}[t]
\centering\small
\setlength{\tabcolsep}{4pt}
\caption{Precision@$m$ of each representation, where $m$ is the admitted count of the row group. Median over 3 repeats of each configuration. Chance is 0.20, and \emph{lower} is better. \dag\ marks a range across repeats above 0.05.}
\label{tab:ablation-precision-p3}
\begin{tabular}{lccccc}
\toprule
model & full params & rand.\ proj & penult. & logits & softmax \\
\midrule
\multicolumn{6}{l}{\emph{Selection:} admits 1 of 10 clients} \\
Fashion-MNIST & 0.000 & 0.000 & 0.000 & 0.000 & 0.000 \\
CIFAR-10 & 0.000 & 0.060 & 0.040 & 0.000 & 0.000 \\
ResNet-9 & 0.000 & 0.000 & 0.020 & 0.000 & 0.010 \\
\midrule
\multicolumn{6}{l}{\emph{Selection:} admits 3 of 10 clients} \\
Fashion-MNIST & 0.000 & 0.003 & 0.000 & 0.000 & 0.000 \\
CIFAR-10 & 0.000 & 0.090 & 0.053\dag & 0.000 & 0.000 \\
ResNet-9 & 0.000 & 0.000 & 0.033 & 0.000 & 0.023 \\
\bottomrule
\end{tabular}
\end{table*}
}

\renewcommand{\AblationSpread}{%
\begin{table}[t]
\centering\small
\setlength{\tabcolsep}{4pt}
\caption{Every cell whose detection AUC ranges by more than 0.05 across repeats, for each of the five measures, with its three repeat values in ascending order.}
\label{tab:ablation-spread-p3}
\begin{tabular}{llcc}
\toprule
model & measure & repeats & range \\
\midrule
\multicolumn{4}{l}{\emph{Selection:} admits 1 of 10 clients} \\
Fashion-MNIST & softmax & 0.681, 0.702, 0.848 & 0.167 \\
CIFAR-10 & rand.\ proj & 0.667, 0.703, 0.744 & 0.077 \\
ResNet-9 & rand.\ proj & 0.927, 0.936, 0.986 & 0.059 \\
ResNet-9 & penult. & 0.890, 0.904, 0.945 & 0.055 \\
ResNet-9 & softmax & 0.622, 0.653, 0.754 & 0.132 \\
\midrule
\multicolumn{4}{l}{\emph{Selection:} admits 3 of 10 clients} \\
CIFAR-10 & rand.\ proj & 0.700, 0.724, 0.750 & 0.0500 \\
CIFAR-10 & penult. & 0.629, 0.752, 0.900 & 0.271 \\
ResNet-9 & penult. & 0.771, 0.876, 0.936 & 0.165 \\
ResNet-9 & softmax & 0.451, 0.551, 0.632 & 0.181 \\
\bottomrule
\end{tabular}
\end{table}
}

\section{Introduction}
\label{sec:intro}
Federated learning keeps raw data on the client, but the updates a client
shares with the server are not inert. An honest-but-curious server holding
them can reconstruct training inputs by gradient inversion
\citep{geiping2020inverting, yin2021see, hatamizadeh2022gradvit}. Secure
aggregation limits that exposure by revealing only the sum of the updates
\citep{bonawitz2017practical}, and differential privacy by perturbing them so
that a curious server \citep{hu2020personalized} or an outside observer
\citep{wei2020federated} learns less from each. Byzantine-robust aggregation
pulls the other way, since the rules that exclude suspect clients compare
clients' parameter vectors pairwise and so need individual updates, which is
exactly what secure aggregation withholds.

Selecting benign clients by parameter distance, as Krum
\citep{blanchard2017machine} and Bulyan \citep{guerraoui2018hidden} do, has
three costs. It exposes clients, since even a participant who sees only the aggregate of the updates
can infer properties of other clients' data \citep{melis2019exploiting}, and a
rule that reads every update individually lets the server learn at least as
much. It is also expensive to compute, since comparing \(n\) clients
pairwise scales as \(O(n^{2} d)\) in a parameter dimension \(d\) that reaches
billions in current large models. Finally, it is restrictive, since a distance
between parameter vectors is defined only when clients share an architecture,
an assumption that heterogeneous deployments do not always meet
\citep{li2019fedmd, ye2023heterogeneous}.

All three costs trace to one decision, that clients are compared on what their
models \emph{are}. {\algoname} compares them on what their models
\emph{compute}. At every round each client draws random inputs, called
pseudo-patterns, for each class, optimises each pattern to raise the logit of
its class under the client's own model, and takes the per-class mean of the
model's logit response to them. We call the result the client's
\emph{fingerprint}. For a task with \(C\) classes it is a \(C \times C\)
matrix, one logit vector per class, so 100 values or 400 bytes in single
precision on a ten-class task, whatever the model behind it. It is computed
locally and sent alongside the client's update, which selection never reads.

So small a fingerprint can serve because Krum and the rules descended from it
admit clients by the \emph{order} of their distance-based scores alone, and an
order can survive a compression that distorts the distances themselves.
Selection therefore asks less of logits than their earlier use in federated
learning for knowledge transfer \citep{li2019fedmd, itahara2021distillation},
which needs them to carry what a model has learned, whereas selection needs
only that they rank clients as a parameter distance would. Closer to this
work, detectors that read model outputs or internal representations
\citep{wang2022flare, rieger2022deepsight, wang2025mitigating} probe with data
or random inputs, whereas {\algoname} optimises its pseudo-patterns against
each client's model and substitutes the distance between fingerprints
for the parameter distance inside an existing selection rule such as Krum
(\Cref{sec:related_work}).

We make three contributions.
\begin{itemize}[nosep]
\item \emph{A selection representation that reads no client parameter.} Each
client's fingerprint holds \(C^{2}\) values whatever the client's model, needs
no clean data, and can be compared across clients whose architectures differ.
\item \emph{Three lightweight variants.} We instantiate {\algoname} as
\emph{Krum-Lite}, \emph{Multi-Krum-Lite} and \emph{Bulyan-Lite}, which replace
the parameter distance of Krum and Bulyan with the fingerprint distance and
change nothing else. We prove that selection on either distance admits no
malicious client when malicious and benign clients are separated by a
sufficient margin (\Cref{lem:order}).
\item \emph{Evidence of parity in robustness.} Through end-to-end
benchmarks of each variant against its parent and an ablation on identical
client updates, we show that every variant holds every attack its parent holds
and that the fingerprint ranks malicious clients on a par with the parameter
vector.
\end{itemize}

This work sets out to establish one thing, that a client-generated fingerprint
of \(C^{2}\) logit values can \emph{substitute} for the full parameter vector
in Byzantine-robust client selection, admitting malicious clients no more often
in total than the parameter distance does. In all our experiments, clients
hold i.i.d.\ partitions, the setting in which the foundational Byzantine-robust
works prove their guarantees and run their evaluations
\citep{blanchard2017machine, guerraoui2018hidden, yin2018byzantine}, and each
cohort is architecturally homogeneous, since models of different architectures
cannot be averaged parameter by parameter, and aggregating them is a separate
problem. Adaptive
adversaries, such as those that report a forged fingerprint or that shape their
update so that its fingerprint still looks benign (\Cref{rem:nullspace}), are
left to future work.

\section{Background and Related Work}
\label{sec:related_work}
Federated averaging \citep{mcmahan2017communication}, still the most widely
used rule, admits every client, and its vulnerability to adversarial ones is
well documented \citep{manzoor2023defending, kumar2023impact}. Krum and
Multi-Krum \citep{blanchard2017machine} and Bulyan \citep{guerraoui2018hidden}
score each client against the others and discard the outliers, and Sniper
\citep{cao2019understanding} keeps the largest group of mutually close updates,
found as a maximum clique. The coordinate-wise median and trimmed mean
\citep{yin2018byzantine} and geometric-median aggregation
\citep{pillutla2022robust} instead combine the updates robustly without
excluding any client. The selection rules differ in their criterion, but they
agree on the quantity they score, the client parameter vector. FLTrust
\citep{cao2020fltrust} stays in parameter space too, training a server update
each round on a small clean training set, the root dataset, and weighting each
client's update by a trust score that falls as the update's direction departs
from the server's.

Reading a model through its responses to chosen inputs is an established idea.
Neural Cleanse \citep{wang2020neural} reverse-engineers a trigger for each
class from clean samples to expose a backdoor, inspecting a single model and
comparing none. Within federated learning, external data has stood in for
classes missing on a client \citep{tijani2021federated}, whereas the
pseudo-patterns of {\algoname} need no data and exist only to draw out the
responses the server compares.

Outputs have also been exchanged in place of parameters. Logits carry
knowledge in distillation \citep{hinton2015distilling}, and federated methods
send them instead of updates, DS-FL to cut communication
\citep{itahara2021distillation} and FedMD so that participants with differently
designed models can keep those designs to themselves \citep{li2019fedmd}. Logits
travel well because their width follows the task rather than the network.

A second line reads outputs and internal representations to tell clients
apart. FLARE \citep{wang2022flare}
weights updates by pairwise discrepancies between penultimate-layer
representations computed on auxiliary data, arguing that defences fixed on
model parameters miss carefully crafted poison. DeepSight
\citep{rieger2022deepsight} clusters clients on differences in the internal
structure and outputs of their updates before filtering. Closest to this work,
USD-FL \citep{wang2025mitigating} derives logits on a reserve dataset, compares
them pairwise under a 1-Wasserstein distance, and flags clients beyond a
threshold.

Taken together, the selection rules agree on comparing parameter vectors, and
the detectors establish that what a model computes can separate benign clients
from malicious ones, but they need data, take their width from the network, or
feed a decision step of their own. {\algoname} differs on each count. Its
pseudo-patterns are optimised against the model and need no data, unlike
Neural Cleanse, FLARE, USD-FL and FLTrust, which need clean samples, auxiliary
data, a reserve dataset and a root dataset respectively. DeepSight also avoids
data, by feeding random inputs to each model, but still reads the output-layer
parameters of every update, which {\algoname} never reads to decide. The
fingerprint of {\algoname} is a class-conditional mean of fixed width \(C \times C\), whatever
the number of pseudo-patterns and whatever the model, where FLARE's
representations take the width of the network's penultimate layer and grow in
number with the data they are computed on. And the distance between
fingerprints is substituted inside an existing selection rule, here Krum and
Bulyan, where USD-FL feeds a threshold, DeepSight a clustering step and FLARE a
trust weight. These differences give selection what the parameter distance
lacks, a comparison that reads no client parameter, compares \(C^{2}\) values
rather than \(d\), and is defined across architectures, while the rule that
consumes it stays unchanged.

\section{Methodology}
\label{sec:methodology}
{\algoname} changes one step of a federated round, the selection of the
clients whose updates are aggregated. It replaces the parameter distance on
which Krum, Multi-Krum and Bulyan rank clients with a distance between
\emph{fingerprints}, summaries of what each client's model computes rather
than of the parameters that compute it. The three variants built this way,
\emph{Krum-Lite}, \emph{Multi-Krum-Lite} and \emph{Bulyan-Lite}, take their
parents' aggregation unchanged, so each differs from its parent in selection
alone. We follow the standard federated setting
\citep{kairouz2021advances}. Each round the server draws \(n\) clients, at most
\(f\) of them Byzantine, and each trains the broadcast model, computes its
fingerprint and returns it with its parameter update \(\mathcal{V}_i \in
\mathbb{R}^{d}\). The server selects on the fingerprints, among the clients of
that round only, and aggregates the selected clients' parameters as the parent
rule does. Symbols are collected in \Cref{tab:symbol-defs}.

\begin{table}[ht]
\centering
\small
\caption{Notation.}
\label{tab:symbol-defs}
\begin{tabular}{l|p{0.62\linewidth}}
\hline
\textbf{Symbol} & \textbf{Definition} \\
\hline
\(n, f\)                 & Clients drawn per round, at most \(f\) of them Byzantine \\
\(d, \mathcal{V}_i\)     & Parameter dimension, client \(i\)'s parameter vector in \(\mathbb{R}^{d}\) \\
\(C\)                    & Number of classes \\
\(\mathcal{M}\)          & A client's trained local model \\
\(N, \mathbf{P}_{j,k}\)  & Patterns per class, and pattern \(j\) of class \(k\) \\
\(z_{j,k}, \mathbf{z}_{j,k}, \bar{\mathbf{z}}_{k}\) & Class-\(k\) logit, full response in \(\mathbb{R}^{C}\), and class mean \\
\(\bar{Z}_i\)            & Client \(i\)'s fingerprint, in \(\mathbb{R}^{C \times C}\) \\
\(S(i), m, n'\)          & Client \(i\)'s selection score, clients admitted, and clients scored at the last pass \\
\(\theta, \beta\)        & Number of first-stage survivors and per-coordinate average size, both Bulyan arms \\
\(\phi, \Pi, p, \epsilon\) & A representation, a distance-preserving map into \(p\) dimensions, and its distortion \\
\(\gamma, \Delta\)       & Benign to Byzantine separation, and largest benign distance \\
\hline
\end{tabular}
\end{table}

\subsection{Building the Fingerprint}
\label{sec:Pseudo Pattern Generation}
\label{sec: Pseudo Pattern Optimization}
\label{sec: Logits Generation}

Each client probes its trained local model \(\mathcal{M}\) with
pseudo-patterns, inputs drawn from an ordinary random number generator with no
reference to any data. For a task over \(C\) classes, the client draws \(N\)
patterns per class, \(\mathbf{P}_{j,k} \sim \mathcal{N}(0, 1)\) in the input
shape of \(\mathcal{M}\), and assigns pattern \(\mathbf{P}_{j,k}\) to class
\(k\). In our experiments the generator is seeded alike for every client and
every round, so all clients start from the same patterns.

The client then refines the patterns into proxies for their classes, adjusting
each \(\mathbf{P}_{j,k}\) so that the trained model's logit for class \(k\) is
as large as possible. It minimises, by gradient descent with momentum,
\[
\mathcal{L} = -\frac{1}{CN} \sum_{k=1}^{C} \sum_{j=1}^{N} z_{j,k},
\qquad z_{j,k} = \big(\mathcal{M}(\mathbf{P}_{j,k})\big)_{k}.
\]
The optimisation drives each pattern toward the behaviour its class elicits,
so that the responses read next can separate clients whose models have learned
different behaviour. Its hyperparameters are given in \Cref{sec:app-setup}.

The refined patterns are then read through the model. For each pattern the
client keeps the complete logit response \(\mathbf{z}_{j,k} =
\mathcal{M}(\mathbf{P}_{j,k}) \in \mathbb{R}^{C}\), rather than the single
component \(z_{j,k}\) that the objective maximises, and averages the \(N\)
responses of each class into a class mean. The \(C\) class means, stacked,
form the client's \emph{fingerprint},
\[
\bar{\mathbf{z}}_{k} = \frac{1}{N} \sum_{j=1}^{N} \mathbf{z}_{j,k},
\qquad
\bar{Z} = \big( \bar{\mathbf{z}}_{1}, \dots, \bar{\mathbf{z}}_{C} \big)
\in \mathbb{R}^{C \times C}.
\]
Keeping the whole response is why the fingerprint holds \(C^{2}\) values
rather than \(C\), and averaging within each class is why \(N\) never reaches
its size. We use \(N = 10\), a choice \Cref{sec:app-setup} explains.

Since neither dimension of \(\bar{Z}\) depends on \(N\) or on the network,
a client is described by \(C^{2}\) values however deep or wide its model,
against \(d\) values for its parameter vector. Two clients therefore remain
comparable even when their parameter vectors differ in length, provided they
share a label space and an input shape, although this work evaluates
architecturally homogeneous cohorts only.

\subsection{Selection on Fingerprints}
\label{sec:Applications}
\label{subsec: Aggregation Function}

Selection scores each client by its distances to the others. With
fingerprints \(\bar{Z}_1, \dots, \bar{Z}_n\) for the \(n\) clients of a round,
client \(i\) scores
\[
S(i) = \sum_{i \to j} \big\| \bar{Z}_i - \bar{Z}_j \big\|_F^{2},
\]
where \(i \to j\) means that \(\bar{Z}_j\) is among the \(n - f - 2\)
fingerprints nearest to \(\bar{Z}_i\), so that of its \(n - 1\) distances a
client discards the largest \(f + 1\). Summing over every other client instead
fails once \(f \geq 2\), since colluders placed far away drag the minimiser
toward another colluder \citep{blanchard2017machine}, whatever vector the
distance is computed on. The neighbourhood is Krum's own, and it needs only a
known bound \(f\) with \(n > 2f + 2\).

Selection then admits the client with the lowest score. To admit \(m\)
clients, the published Multi-Krum takes the \(m\) lowest scores from a single
sort \citep{blanchard2017machine}, while the extended version of that paper
\citep{blanchard2017byzantine} selects one client, removes it and re-scores
the rest, provided \(n - m > 2f + 2\). Blanchard et al. prove resilience for
that iterative form and leave the single sort open. Both arms of our experiments, parameter
distance and fingerprint distance alike, use the iterative form.
\emph{Krum-Lite} admits \(m = 1\) client, and \emph{Multi-Krum-Lite} admits
\(m = n - 2f - 3\), the largest \(m\) that condition allows, which is 3 at
\(n = 10\) and \(f = 2\). Algorithm~\ref{algorithm1} gives the procedure.

Bulyan selects survivors by iterated Krum and then filters them coordinate by
coordinate \citep{guerraoui2018hidden}. \emph{Bulyan-Lite} uses the same
selection as its first stage, run for \(\theta = n - 2f - 2\) clients, two
fewer than the published \(\theta = n - 2f\), in both Bulyan arms, so that
every pass of the iterative selection still scores more than \(2f + 2\)
clients, as Krum's score needs. This runs one pass beyond Multi-Krum's
condition, since \(n - \theta = 2f + 2\).

Their proof for Krum and the iterative Multi-Krum holds under the conditions
they state, which include identically distributed honest updates
\citep{blanchard2017byzantine, blanchard2017machine}, but it rests on
parameter distances, so it transfers to a variant directly when the
two distances select the same clients in every round, since the output is then
identical. What this work establishes is \Cref{lem:order}, which gives the
separation under which selection admits no Byzantine client, on parameters and
on fingerprints alike.

\begin{algorithm}[t]
    \caption{{\algoname} selection for Krum-Lite ($m=1$) and Multi-Krum-Lite ($m>1$), then the parent's aggregation. Bulyan-Lite runs the same selection with $m=\theta$ and replaces the average with Bulyan's coordinate-wise filter. Algorithm~\ref{algorithm-full} in \Cref{sec:app-algorithm} adds the client side.}
    \label{algorithm1}
    \begin{algorithmic}[1]
        \item[] \emph{// server side, selection}
        \STATE \textbf{Input:} fingerprints $\bar{Z}_1, \dots, \bar{Z}_n$, bound $f$, clients to admit $m \leq n - 2f - 2$ \COMMENT{every pass scores $|R| > 2f + 2$}
        \STATE $\delta_{ij} \gets \| \bar{Z}_i - \bar{Z}_j \|_F^2$ for all $i \neq j$ \COMMENT{computed once}
        \STATE $R \gets \{1, \dots, n\}$, \ $I \gets \emptyset$ \COMMENT{remaining, admitted}
        \FOR{$s = 1, \dots, m$}
            \FOR{each $i \in R$}
                \STATE $S(i) \gets$ sum of the $|R| - f - 2$ smallest $\delta_{ij}$, $j \in R \setminus \{i\}$
            \ENDFOR
            \STATE $i_* \gets \arg\min_{i \in R} S(i)$
            \STATE $R \gets R \setminus \{i_*\}$, \ $I \gets I \cup \{i_*\}$
        \ENDFOR
        \item[] \emph{// server side, the parent's aggregation, unchanged}
        \STATE \textbf{Return:} $\frac{1}{m} \sum_{i \in I} \mathcal{V}_i$ \COMMENT{parameters read here only}
    \end{algorithmic}
\end{algorithm}

\subsection{Aggregation, Unchanged from the Parents}
\label{subsec: Bulyan-Lite}

The selected clients are then aggregated exactly as the parent rule aggregates
them. \emph{Krum-Lite} returns the one admitted parameter vector, and
\emph{Multi-Krum-Lite} returns the average of the \(m\) admitted vectors,
which is Multi-Krum's own averaging step \citep{blanchard2017machine}.
\emph{Bulyan-Lite} applies Bulyan's coordinate-wise filter to the \(\theta\)
survivors and to them only, so the filter contributes to robustness exactly as
in Bulyan. For each coordinate it averages the \(\beta\) values nearest the
survivors' coordinate-wise median \citep{guerraoui2018hidden}, with \(\beta =
\theta - 2f = n - 4f - 2\) following from our \(\theta\). A \(\beta\) of at
least two requires \(n \geq 4f + 4\), so in the experiments Bulyan and
\emph{Bulyan-Lite} draw 12 clients per round where every other rule draws 10.

Aggregation is the step that still reads client parameters, so {\algoname},
which changes selection only, is not a standalone end-to-end protocol.
Selection reads no client parameter, so gradient inversion, which needs parameters or
gradients, has nothing to work on at that step, but every client's parameters
are still uploaded for aggregation. Secure aggregation
\citep{bonawitz2017practical} and differential privacy \citep{wei2020federated,
hu2020personalized} protect them against different threats and are often
applied together. Secure aggregation stays available to any aggregation rule
that reads only sums, such as averaging or the geometric median computed as a
sequence of secure averages \citep{pillutla2022robust}.

\subsection{Selection Cost}

Since aggregation is unchanged, any computational saving lies in selection,
where the following bound applies.

\begin{lemma}[Selection cost]
\label{lem:complexity}
Selecting \(m\) clients on fingerprints runs in \(O(n^{2} C^{2} + m\, n^{2}
\log n)\) time, against \(O(n^{2} d + m\, n^{2} \log n)\) for the
parameter-distance rule it replaces.
\end{lemma}

The count behind the bound is given in \Cref{sec:app-wallclock}. The bound is
an improvement whenever \(C^{2} \ll d\). For the 10-class tasks of this work
\(C^{2} = 100\), against \(d\) of 29{,}130, 552{,}362 and 4{,}909{,}002 for the
three architectures, so the compared quantity shrinks by factors of 291,
5{,}524 and 49{,}090 respectively. The saving is in elements compared, for the
selection step alone. The client pays instead, with one pattern optimisation
per round, whose cost \Cref{subsec:cost} reports.
\Cref{tab:krum-vs-krum-lite} sets the selection cost beside the other
properties that separate the variants from their parents.

\krumVsKrumLite

\section{Theoretical Analysis}
\label{sec:theory}
Selecting on $C^{2}$ values in place of $d$, with $C^{2} \ll d$ for every model
considered here, raises two questions. The first is
whether a distance computed in so small a space can order clients at all. The
second is what the logit map provides that any compression of the same size
would not.

The first question is answered in principle by a classical result of Johnson
and Lindenstrauss \citep{johnson1984extensions, dasgupta2003elementary}. For
any \(n\) points there is a map into \(p\) dimensions that distorts every
pairwise squared distance by a factor of at most \(1 \pm \epsilon\), with \(p\)
of order \(\epsilon^{-2} \log n\) (\Cref{lem:jl} in \Cref{sec:app-theory}).
The bound involves $n$, the number of clients compared in a round, and not
$d$, so the size of the model alone does not rule out a representation of
fixed width. The lemma is an existence result, proved with a random
linear projection. It bears on whether $C^{2}$ values could ever suffice, not
on whether the fingerprint, which is neither linear nor a random projection,
is a map that achieves it. At our cohort sizes its guarantee at
\(p = C^{2} = 100\) is weak (\Cref{sec:app-theory}), but a weak distortion
need not be fatal, because Krum and its variants admit clients by the order of
their scores rather than by their values, so what matters is whether a margin
between benign and Byzantine clients survives the distortion. We first make
the relevant notion of separation precise.

\begin{definition}[$\gamma$-separability]
\label{def:sep}
Let $\phi$ be the representation on which clients are compared. A round is
$\gamma$-separable under $\phi$ if there exists $\Delta > 0$ such that
$\|\phi(\mathcal{V}_i) - \phi(\mathcal{V}_j)\| \leq \Delta$ for every pair of
benign clients $i$ and $j$, and $\|\phi(\mathcal{V}_i) - \phi(\mathcal{V}_a)\|
\geq (1+\gamma)\Delta$ for every benign client $i$ and every Byzantine client
$a$, with the Frobenius norm for fingerprints. Taking $\phi$ to be the identity gives separability in parameter space,
and taking $\phi$ to map each client's parameters to its fingerprint gives it in
fingerprint space.
\end{definition}

Every selection rule here scores clients with Krum's neighbourhood score, run
iteratively as in \Cref{subsec: Aggregation Function}. The result below
concerns that selection alone, so Bulyan's coordinate-wise second stage lies
outside it.

\begin{lemma}[Separation keeps attackers out]
\label{lem:order}
Let a round with at most $f$ Byzantine clients be $\gamma$-separable under
$\phi$, and let selection on $\phi$ admit $m$ clients iteratively, so that its last pass scores $n' = n - m + 1$
clients, with $n' > 2f + 1$ so that no Byzantine client can fill its
neighbourhood with its colluders. If
\[
(1+\gamma)^{2} > \frac{n'-f-2}{n'-2f-1},
\]
then every client admitted is benign.
\end{lemma}

At a pass that scores \(r\) clients, a benign client can fill its
neighbourhood with benign peers and so scores at most \((r-f-2)\Delta^{2}\),
while a Byzantine client must count at least \(r-2f-1\) benign clients among
its nearest and scores at least \((r-2f-1)(1+\gamma)^{2}\Delta^{2}\). The
condition puts every Byzantine score above every benign one at the last pass,
where the ratio of the two counts is largest, and so at every pass. The full
proof is in \Cref{sec:app-theory}. Unlike the parents' resilience proofs,
\Cref{lem:order} assumes no distribution of honest updates, only their
separation in the round.

The lemma also covers a compression of the parameters. A map that distorts
squared distances by at most \(1 \pm \epsilon\) keeps the round separable, with
\((1+\gamma)^{2}\) reduced by the factor \((1-\epsilon)/(1+\epsilon)\)
(\Cref{cor:compression} in \Cref{sec:app-theory}). Fingerprints are compared
directly, with no distortion term, and with $f = 2$ \Cref{lem:order} asks for
\(\gamma > 0.10\) at the single pass of \emph{Krum-Lite} (\(n = 10\)), \(0.16\)
at the last pass of \emph{Multi-Krum-Lite} (\(n' = 8\)) and \(0.23\) at the
last pass of \emph{Bulyan-Lite} (\(n = 12\), \(n' = 7\)), rounded up. Below
these thresholds the lemma no longer gives the guarantee, and the margin
vanishes altogether as benign clients drift as far from one another as they
sit from the attackers.

The argument so far would license any map that preserves separation, including
a random projection of the parameters, which \Cref{subsec:ablation} scores as a
control at $C$ dimensions rather than $C^{2}$. What distinguishes the
fingerprint is that it evaluates the function the model computes rather than
compressing the parameters that compute it, so an update that changes the
model's class-averaged responses to the pseudo-patterns changes the
fingerprint, whereas a projection preserves only whatever separation already
exists. How well the pseudo-patterns cover the behaviour an attacker actually
changes is not established here.

\begin{remark}
\label{rem:nullspace}
The map from parameters to fingerprints sends $\mathbb{R}^{d}$ to
$\mathbb{R}^{C^{2}}$, so near any model where the map is differentiable,
parameter changes in a subspace of dimension at least $d - C^{2}$, more than
four million for ResNet-9, leave the fingerprint unchanged to first order. An
adversary who maximises attack success within that subspace, so that its
fingerprint still looks benign, is the natural adaptive attack against this
method. We do not evaluate such an adversary in this work, and we return to it
in \Cref{sec:conclusion}.
\end{remark}

\section{Experiments}
\label{sec:experiments}
The benchmark asks whether the substitution costs robustness in end-to-end
training, each variant against its own parent. The ablation asks whether the
fingerprint ranks clients as the parameter vector does on identical updates,
and how it compares with a smaller random projection, which bears on, but does
not settle, the second question of \Cref{sec:theory}. The cost analysis
measures what the substitution trades.

We train on Fashion-MNIST \citep{xiao2017fashion} with a CNN and on CIFAR-10
\citep{krizhevsky2009learning} with a CNN and with ResNet-9
\citep{he2016deep}, under i.i.d.\ partitions, the protocol of the foundational
works cited in \Cref{sec:intro}. The population is 100 clients, 25 of them
adversarial, and each round a seeded rule draws \(n\) clients with exactly
\(f = 2\) adversaries among them. Every rule draws \(n = 10\), a 20\% share,
except the Bulyan pair, which draws 12 and faces 16.7\%
(\Cref{subsec: Bulyan-Lite}), and rules drawing the same \(n\) see the same
clients in the same round. Four attacks are run. The 1-pixel and
9-pixel backdoors plant a trigger of one pixel \citep{gu2017badnets} or nine
sparse pixels and relabel the triggered images to a target class, label
flipping relabels four source classes to four targets, and the targeted attack
is the single-pair version \citep{tolpegin2020data}. Federated
averaging, which admits every client, is the undefended reference, and the
coordinate-wise median and the trimmed mean \citep{yin2018byzantine} are
references that combine every update without selecting. Full details are in
\Cref{sec:app-setup}.

Each round the global model is evaluated on the test set for clean accuracy
and on a held-out set of attacked samples for attack success, which is the
backdoor accuracy under a pixel trigger and the malicious-task accuracy under
a label attack, and we count the rounds in which an adversary is admitted.
Every benchmark accuracy is the mean over the final 10 of 100 rounds, and each
configuration is summarised by the median of three repeats of one fixed seed,
which differ only by run-to-run non-determinism (\Cref{sec:app-setup}),
\(\pm\) the median of their within-run standard deviations. One rule in one
configuration is a \emph{cell}. A
configuration is \emph{discriminative} when the attack exceeds 40\% success
under federated averaging and that run reaches at least 40\% clean
accuracy, since where an attack fails with no defence in place there is
nothing to resist. Eleven of the twelve configurations, three dataset and
model pairs under four attacks, are discriminative, and every comparison below
is computed on those eleven, except the admission counts, which cover all
twelve. The twelfth, the targeted attack on CIFAR-10 with ResNet-9, where
undefended averaging reaches 29.1\%, stays in the tables, marked.

\subsection{Results}
\label{subsec:results}

\Cref{tab:ExperimentalResults} reports every cell and \Cref{tab:ResultsDelta}
the paired differences in both metrics. The substitution does not cost
robustness. All three variants hold the attack below the 40\% success
threshold in all eleven configurations, and so do Krum and Multi-Krum, while
only Bulyan falls short, holding in 10 of 11. The paired median differences in
attack success are small in all three pairs (\Cref{tab:ResultsDelta}, and per
configuration \Cref{fig:gap-parent} in \Cref{sec:app-results}), and no variant in any discriminative configuration lets attack
success past 20\%, the highest being \emph{Bulyan-Lite} at 17.41\% under the
targeted attack on CIFAR-10 with the CNN. Under the targeted attack
\emph{Bulyan-Lite} ends above its parent (\Cref{fig:rounds-main}c in
\Cref{sec:app-rounds}).

\ExperimentalResults
\ResultsDelta

One configuration separates a pair outright. Under the 1-pixel trigger on
CIFAR-10 with the CNN, Bulyan reaches 52.79\% attack success while
\emph{Bulyan-Lite} holds it at 10.18\%, a difference of 42.6 points and the
widest in the study (\Cref{fig:rounds-main}a). All three of Bulyan's repeats,
which share one seed, lie above the threshold, spanning 8.1 points, and the gap
comes on the study's smallest adversarial share, since the Bulyan pair draws
twelve clients. Nowhere else does one rule of a pair hold the attack while the
other fails. This is one rule on one configuration, not a claim that the
fingerprint detects better than the parameter distance, and the ablation below
finds the two at par.

The admission counts agree (\Cref{tab:Admission} in \Cref{sec:app-results}).
Pooling the three architectures and the three repeats of each attack over the
final 10 rounds, 90 rounds in all, neither Krum pair admits an adversary under
any attack. Bulyan admits an adversary in 24 of 90 rounds under the 1-pixel trigger and
5 of 90 under the 9-pixel, and \emph{Bulyan-Lite} in none, 4 and 1 of 90 under
the 1-pixel, 9-pixel and targeted attacks, the last where Bulyan admits none.

Clean accuracy does not suffer. For the Krum variants, replacing the distance
raises it in every one of the eleven configurations, at median differences of
$+4.22$ and $+3.55$ points, while \emph{Bulyan-Lite} leads its parent in only
5 of 11 and trails at a median of $-0.19$ points, so the substitution is
accuracy-neutral to accuracy-positive rather than uniformly positive. At the
median across configurations every selection rule sits below undefended
averaging, 71.41\% to 81.58\% against 82.16\%, which is the price of
discarding clients and is not specific to this work.

Robustness here comes from selecting clients. The coordinate-wise median and the trimmed mean, which
combine every update, hold only 5 of the 11 configurations and resist neither
pixel trigger, at 85.79\% to 99.97\% attack success
(\Cref{sec:app-nofilter}). Three cells have repeats spanning more than 5
points, and \Cref{sec:app-repeats} draws each one repeat per row, while
\Cref{sec:app-rounds} traces every rule through all 100 rounds.

\subsection{Ablation Study}
\label{subsec:ablation}

Separate training runs answer what a deployment would experience, but not
whether the fingerprint can substitute for the parameter vector in ranking
clients. Two runs diverge after the first round, so a later gap reflects two
distance functions and two training histories at once, and where neither run
admits an attacker a match says nothing about the distance. We therefore score
five candidate measures as detectors on the updates of one run, each with the
score of \Cref{subsec: Aggregation Function}. Only the fingerprint drives
selection and the rest are recorded alongside, so every measure faces an
identical ranking problem. The five are the flattened client update (\(d\)),
which Krum ranks on, a fixed Gaussian projection of it (\(C\)), the
penultimate features (\(h\), their width), the logit fingerprint (\(C^{2}\)),
and that fingerprint after a softmax, with dimensions from \(C = 10\) up to
\(d\) of about 4.9 million. \emph{AUC} is the probability that a randomly
chosen malicious client is ranked more distant than a randomly chosen benign
one, so \(0.5\) is chance, and \emph{precision@\(m\)} is the malicious fraction
of the \(m\) clients selection would admit, so lower is better and chance is
\(f/n = 0.20\). The grid is three architectures by the two Krum variants,
admitting one and three clients of ten, over 100 rounds and three repeats
under the 9-pixel backdoor alone, each run's AUC being the mean of its
per-round values. \Cref{tab:ablation-delta-p3} gives the change in AUC from
the parameter vector to the fingerprint, and \Cref{sec:app-ablation} gives
every measure in full (\Cref{tab:ablation-auc-p3}), the precision table and
the spread across repeats.

\AblationDelta

The fingerprint matches the full parameter vector. Their per-cell medians
differ by at most $0.006$ AUC in either direction and by $0.000$ at the median,
with the fingerprint between $0.988$ and $1.000$ against $0.993$ to $1.000$ for
the parameters, and precision@\(m\) is $0.000$ for both in every cell.
Replacing \(d\) with \(C^{2}\), between two and five orders of magnitude fewer
values, costs nothing on this ranking problem.

The fingerprint is also ahead of a random projection of the parameter vector,
though not a size-matched one. The random projection, ten times smaller than
the fingerprint, trails it in all six cells, by $0.017$ to $0.285$ AUC and by
$0.042$ at the median, and it places malicious clients among those selection
would admit, its precision@\(m\) reaching $0.090$ on CIFAR-10 against $0.000$
throughout for the fingerprint. The softmax trails the fingerprint in every
cell and the penultimate representation in five of six, and only the
fingerprint and the parameter distance stay within $0.008$ AUC across their
three repeats (\Cref{sec:app-ablation}).

\subsection{Cost}
\label{subsec:cost}

The robustness results above establish that the substitution is safe to make
rather than that it is worth making. What makes it worth making is that the
trust decision compares \(C^{2}\) values rather than \(d\), by the factors
given after \Cref{lem:complexity}, without reading a client parameter.
Measured on the clock at \(n = 10\), the selection speedup is far smaller than
the element ratio, and \Cref{sec:app-wallclock} reports both, with
\Cref{tab:overhead} measuring the selection, communication and client cost.

The fingerprint adds 400 bytes per client per round, a share largest on the
smallest model and even there a third of one percent. Each client runs one
pseudo-pattern optimisation per round, which at 1000 iterations costs between
78\% and 217\% of a round of local training. The burden falls as the model
grows. On ResNet-9 the optimisation is already cheaper than the training it
accompanies, and it weighs heaviest on the Fashion-MNIST CNN, where there is
least training to amortise it against. The cost falls on clients while the
saving is at the server and in what its trust decision reads, so this is a
shift in where the work is done and in what selection discloses rather than a
reduction in total work. The shift also leaves the method indifferent to the
size of the federation. The client's optimisation is local and runs in
parallel, so it grows with neither the cohort nor the population, and the
server compares the cohort it draws each round rather than the enrolled
population, so at a fixed cohort size a deployment of millions costs the same
per round as one of hundreds.

The iteration count, held uniform here rather than tuned per task, is the
largest lever on this cost, since cost is linear in it, and in earlier
experiments 100 iterations performed comparably on the two CNN architectures.
The patterns per class are a second, since the fingerprint's width and payload
do not change with $N$ while the optimisation scales with it.

\section{Conclusion}
\label{sec:conclusion}
We introduced {\algoname}, which substitutes a \emph{fingerprint}, a client's
class-conditional logit response to pseudo-patterns optimised against its
model, for the client's parameters in Byzantine-robust client selection, and
instantiated it as \emph{Krum-Lite}, \emph{Multi-Krum-Lite} and
\emph{Bulyan-Lite}, which differ from Krum, Multi-Krum and Bulyan in the
distance alone. Each variant matched its parent, holding every attack the
parent held, and on identical updates the fingerprint ranked malicious clients
as well as the parameter vector. By construction, the trust decision reads no
client parameter, compares \(C^{2}\) values instead of \(d\), and remains defined
when clients' architectures differ.

\textbf{Limitations.} All results are i.i.d., the setting in which Krum and
Bulyan are proved resilient \citep{blanchard2017machine, guerraoui2018hidden}.
Attacks that exploit heterogeneous data are known to defeat Krum
\citep{karimireddy2022bucketing}, and whether fingerprints keep malicious
clients apart under such data is untested. That question, adaptive adversaries
such as that of \Cref{rem:nullspace}, and cohorts that mix architectures are
the work we plan to pursue next.

\clearpage
\bibliographystyle{icml2026}
\bibliography{6-references}

\clearpage
\appendix
\onecolumn
\section*{Appendix}
\section{Full Selection Procedure}
\label{sec:app-algorithm}

Algorithm~\ref{algorithm-full} extends Algorithm~\ref{algorithm1} with the
client side, the computation by which each client turns its trained model into
a fingerprint (\Cref{sec:Pseudo Pattern Generation}).

\begin{algorithm}[h]
    \caption{The full {\algoname} procedure, server and client side, for Krum-Lite ($m=1$) and Multi-Krum-Lite ($m>1$), followed by the parent's aggregation. Bulyan-Lite's first stage is the same selection with $m=\theta$, after which Bulyan's coordinate-wise filter replaces the average. Sorting the scores once and taking the $m$ lowest, the published Multi-Krum, is not used here.}
    \label{algorithm-full}
    \begin{algorithmic}[1]
        \item[] \emph{// server side, selection}
        \STATE \textbf{Input:} fingerprints $\bar{Z}_1, \dots, \bar{Z}_n$, bound $f$, clients to admit $m \leq n - 2f - 2$ \COMMENT{every pass scores $|R| > 2f + 2$}
        \STATE $\delta_{ij} \gets \| \bar{Z}_i - \bar{Z}_j \|_F^2$ for all $i \neq j$ \COMMENT{computed once}
        \STATE $R \gets \{1, \dots, n\}$, \ $I \gets \emptyset$ \COMMENT{remaining, admitted}
        \FOR{$s = 1, \dots, m$}
            \FOR{each $i \in R$}
                \STATE $S(i) \gets$ sum of the $|R| - f - 2$ smallest $\delta_{ij}$, $j \in R \setminus \{i\}$
            \ENDFOR
            \STATE $i_* \gets \arg\min_{i \in R} S(i)$
            \STATE $R \gets R \setminus \{i_*\}$, \ $I \gets I \cup \{i_*\}$
        \ENDFOR
        \item[] \emph{// server side, the parent's aggregation, unchanged}
        \STATE \textbf{Return:} $\frac{1}{m} \sum_{i \in I} \mathcal{V}_i$ \COMMENT{parameters read here only}
    \end{algorithmic}

    \hfill

    \begin{algorithmic}[1]
        \item[] \emph{// client side}
        \item[] \textbf{ClientFingerprint($\mathcal{M}_i$):}
        \STATE \textbf{Input:} $\mathcal{M}_i$ \COMMENT{trained local model}
        \STATE $\mathbf{P} \gets$ $N$ seeded patterns per class \COMMENT{\Cref{sec:Pseudo Pattern Generation}}
        \STATE $\mathbf{P} \gets$ minimise $\mathcal{L}$ over $\mathbf{P}$ under $\mathcal{M}_i$
        \STATE $\bar{Z}_i \gets$ class-wise mean of $\mathcal{M}_i(\mathbf{P})$ \COMMENT{$C$ means of $C$ values}
        \STATE \textbf{Return:} $\bar{Z}_i$, $\mathcal{V}_i$ \COMMENT{$\mathcal{V}_i$ from $\mathcal{M}_i$}
    \end{algorithmic}
\end{algorithm}

\section{Theoretical Details}
\label{sec:app-theory}

For a distortion $\epsilon$, Johnson and Lindenstrauss~\citeyearpar{johnson1984extensions} bound the
target dimension by an order of $\epsilon^{-2} \log n$ with an unspecified constant, and the explicit
form below is that of Dasgupta and Gupta~\citeyearpar{dasgupta2003elementary}.

\begin{lemma}[Johnson--Lindenstrauss, restated]
\label{lem:jl}
For any $0 < \epsilon < 1$, any integer $n$, and any set $V$ of $n$ points in
$\mathbb{R}^{d}$, let $p$ be a positive integer with
\[
p \geq 4\left(\frac{\epsilon^{2}}{2} - \frac{\epsilon^{3}}{3}\right)^{-1} \ln n .
\]
Then there exists a map $\Pi : \mathbb{R}^{d} \rightarrow \mathbb{R}^{p}$ such
that, for all $u, v \in V$,
\[
(1-\epsilon)\|u-v\|^{2} \leq \|\Pi(u)-\Pi(v)\|^{2} \leq (1+\epsilon)\|u-v\|^{2} .
\]
\end{lemma}

At the cohort sizes used here and $\epsilon = 0.5$ it asks for
$p \geq 111$ at $n = 10$ and $p \geq 120$ at $n = 12$, both slightly above the
$C^{2} = 100$ available with 10 classes, so the guarantee it gives directly is
weak, at $\epsilon \approx 0.535$ and $0.565$ respectively.

\begin{proof}[Proof of \Cref{lem:order}]
Consider a pass that scores $r \geq n'$ clients with every Byzantine client
among them, as at the first pass. Separation is a statement about pairs, so it
holds among these clients too. A score sums the $r-f-2$ smallest squared
distances, so a benign client, which can fill its neighbourhood with benign
peers, scores at most $(r-f-2)\Delta^{2}$. \Cref{def:sep} leaves Byzantine to
Byzantine distances unbounded, so colluders may fill at most $f-1$ of a
Byzantine client's summed terms, leaving at least $r-2f-1 \geq 1$ against benign
clients and a score of at least $(r-2f-1)(1+\gamma)^{2}\Delta^{2}$. The ratio
$(r-f-2)/(r-2f-1)$ does not increase with $r$, so the stated condition makes the
second bound exceed the first at every pass. Each pass therefore admits a benign
client and leaves every Byzantine client for the next.
\end{proof}

\begin{corollary}[A compression keeps separation]
\label{cor:compression}
If a round is $\gamma$-separable under the identity and a map $\Pi$ satisfies
the distortion bound of \Cref{lem:jl} on the round's clients, then the
conclusion of \Cref{lem:order} holds for selection on $\Pi$ whenever
\[
(1+\gamma)^{2} > \frac{n'-f-2}{n'-2f-1} \cdot \frac{1+\epsilon}{1-\epsilon}.
\]
\end{corollary}

\begin{proof}
After the map, benign pairs lie within $\sqrt{1+\epsilon}\,\Delta$ of each
other and benign to Byzantine pairs at least $\sqrt{1-\epsilon}\,(1+\gamma)\Delta$
apart, which is $\gamma'$-separability with $(1+\gamma')^{2} =
(1+\gamma)^{2}(1-\epsilon)/(1+\epsilon)$. Substituting into \Cref{lem:order}
gives the condition.
\end{proof}

With $f = 2$, a compression of the parameters with the distortion of
\Cref{lem:jl} at $p = 100$ would need, by \Cref{cor:compression}, $\gamma >
1.00$ at the single pass of \emph{Krum-Lite} ($n = 10$, $\epsilon \approx
0.535$), $\gamma > 1.10$ at the last pass of \emph{Multi-Krum-Lite} ($n' =
8$), and $\gamma > 1.33$ at the last pass of \emph{Bulyan-Lite} ($n = 12$,
$n' = 7$, $\epsilon \approx 0.565$). These thresholds belong to a
hypothetical map of that kind, which the fingerprint is not, and they are not
directly comparable with the fingerprint-space thresholds of
\Cref{sec:theory}. All thresholds are rounded up.

\section{Experimental Setup}
\label{sec:app-setup}

\paragraph{Attacks.} \Cref{sec:experiments} defines the four attacks. The
1-pixel trigger is the single-pixel backdoor of BadNets~\citep{gu2017badnets},
and the 9-pixel trigger, our variant, places nine sparse pixels within a
$5\times3$ bounding box at a fixed location. Both are trained to leave
untriggered inputs correctly classified. Label flipping relabels four
source-target class pairs and the targeted attack one, following Tolpegin et
al.~\citeyearpar{tolpegin2020data}. Neither label attack plants a trigger, so any
effect it has on a client's fingerprint comes through the class-conditional
responses it shifts. Malicious updates are neither boosted nor scaled. These
attacks draw on work on backdoors in centralised
training~\citep{liu2018trojaning, li2020invisible, bagdasaryan2021blind}, their
federated counterparts~\citep{xie2020dba}, and model poisoning with explicit
boosting~\citep{bhagoji2019analyzing}.

\paragraph{Architectures.} The Fashion-MNIST CNN comprises two convolutional
layers with $5\times5$ kernels, mapping from 1 to 16 channels and from 16 to
32, each followed by BatchNorm and a $2\times2$ max-pool, and a fully connected
output layer of size 10. The CIFAR-10 CNN begins with two $3\times3$
convolutional layers, from 3 to 32 channels and then 32 to 32, followed by a
$2\times2$ max-pooling layer, with BatchNorm after every convolution. The
pattern repeats twice more, from 32 to 64 and 64 to 64, then from 64 to 128 and
128 to 128, each followed by a $2\times2$ max-pool, and a fully connected layer
of size 128 precedes the 10-dimensional output layer. ResNet-9 is a residual
network~\citep{he2016deep} of nine layers, an initial convolution followed by
four basic residual blocks of two convolutions each. The parameter dimension
\(d\) of each model counts every floating-point tensor, BatchNorm statistics
included.

\paragraph{Hyperparameters.} Local training runs 10 epochs with a batch size
of 32, using SGD~\citep{bottou2010large} at a learning rate of 0.01 and a
momentum of 0.1. The pattern optimisation of
\Cref{sec:Pseudo Pattern Generation} runs gradient descent with momentum 0.99
at a learning rate of 0.9, for 1000 iterations, on the raw logits rather than
the softmax, with the model in evaluation mode and the patterns of all classes
optimised together. Nothing else bounds it, since the patterns are not
projected back into a valid input range. Each client draws \(N = 10\) patterns
per class. Each component of a class mean has \(1/N\) of the variance of a
single response, so more patterns steady the fingerprint, but the variance
falls only as \(1/N\) while the cost of optimising the patterns grows
linearly in \(N\), and \(N = 10\) proved sufficient for stable detection
across all our experimental settings.

\paragraph{Evaluation protocol.} Benchmark accuracies average the final 10 of
100 rounds, since a single round is a noisy sample and the early rounds precede
convergence. The seed is held fixed, so a repeat re-runs a byte-identical
configuration and the spread between repeats measures non-determinism of
implementation and hardware rather than sampling variability.

\paragraph{Implementation.} Every algorithm and architecture is implemented in
PyTorch~\citep{paszke2019pytorch}, starting from the FLTK
codebase\footnote{\url{https://github.com/bacox/fltk}}. The trimmed mean drops
the four largest and four smallest values of each coordinate at \(n = 10\).
The repeats ran on NVIDIA RTX A4000, A100 and V100 accelerators, so timings are
reported for one accelerator rather than pooled.

\section{Additional Benchmark Results}
\label{sec:app-results}

\Cref{fig:gap-parent} draws the paired attack-success differences of
\Cref{tab:ResultsDelta} one configuration at a time, and \Cref{tab:Admission}
gives the admission counts that \Cref{subsec:results} reports.

\begin{figure}[h]
    \centering
    \includegraphics[width=0.6\linewidth]{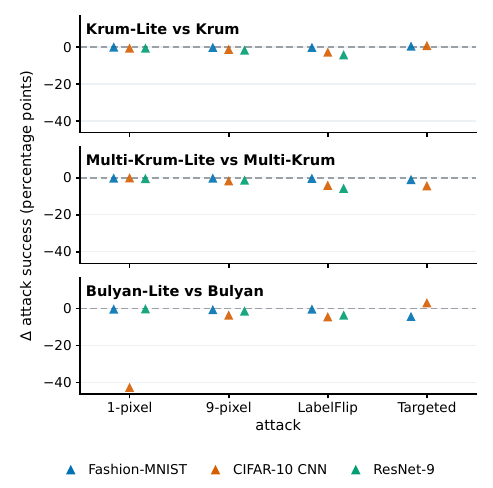}
    \caption{Each variant minus the parameter-distance rule it replaces, in
    attack success, one point per configuration. Negative is better for the
    variant. Only one point of 33 sits materially above zero, and the single
    large negative one is \emph{Bulyan-Lite} against Bulyan under the
    1-pixel trigger on CIFAR-10 with the CNN. Differences are in percentage
    points, the difference between two percentages rather than a relative change.}
    \label{fig:gap-parent}
\end{figure}

\AdmissionTable

\section{Rules That Combine Every Update}
\label{sec:app-nofilter}

\Cref{tab:NoFilterBaselines} reports every cell for the coordinate-wise median
and the trimmed mean, the two references that combine every update without
selecting, which \Cref{subsec:results} summarises. They are kept out of the
main results table, which compares rules that select clients.

\NoFilterBaselines

\section{Repeats That Disagree}
\label{sec:app-repeats}

Three of the 108 cells across the two results tables have repeats spanning
more than 5 points, all on CIFAR-10 with the CNN under a pixel trigger, namely
\emph{Bulyan-Lite} at 9-pixel and Bulyan at both. The 1-pixel gap of
\Cref{subsec:results}, the only claim that turns on such a cell, survives all
three repeats.
\Cref{fig:repeat-grid} draws the three daggered cells one repeat per row, so
what each median averages over can be seen.

\begin{figure*}[t]
    \centering
    \includegraphics[width=\textwidth]{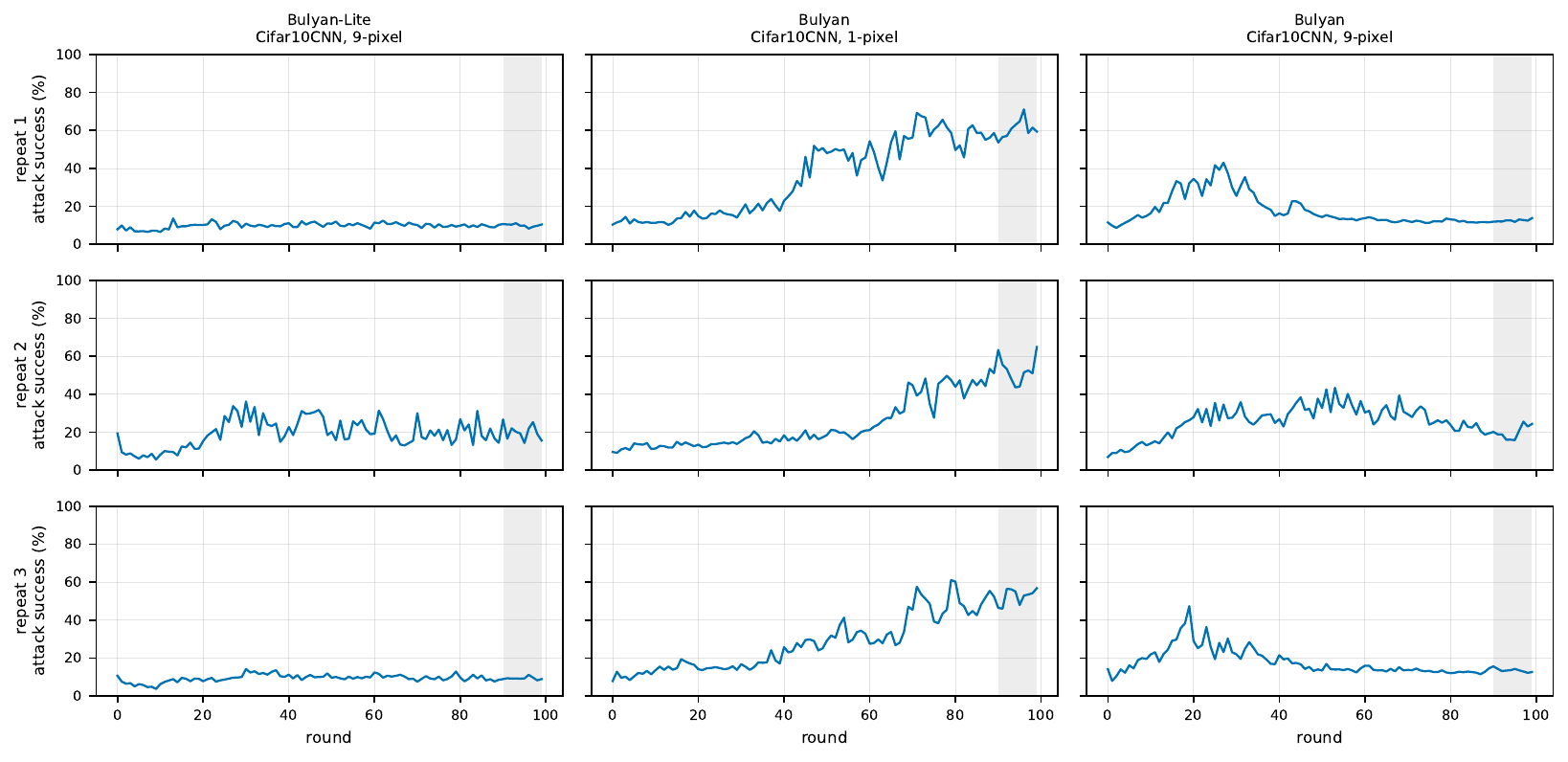}
    \caption{The three cells whose repeats disagree by more than 5 points, all
    on CIFAR-10 with the CNN, one column each and one repeat per row. The
    shaded band is the final 10 rounds. Under the 1-pixel trigger all three
    Bulyan repeats rise above the 40\% threshold, spanning 8.1 points. In the
    \emph{Bulyan-Lite} 9-pixel column the second repeat settles near 20\% where
    the other two stay near 10\%, and in the Bulyan 9-pixel column all three
    rise and fall but at different heights.}
    \label{fig:repeat-grid}
\end{figure*}

\clearpage
\section{Per-Round Behaviour}
\label{sec:app-rounds}

\Cref{fig:rounds-main} picks out the three trajectories that
\Cref{subsec:results} argues from.

\begin{figure*}[t]
    \centering
    \includegraphics[width=\textwidth]{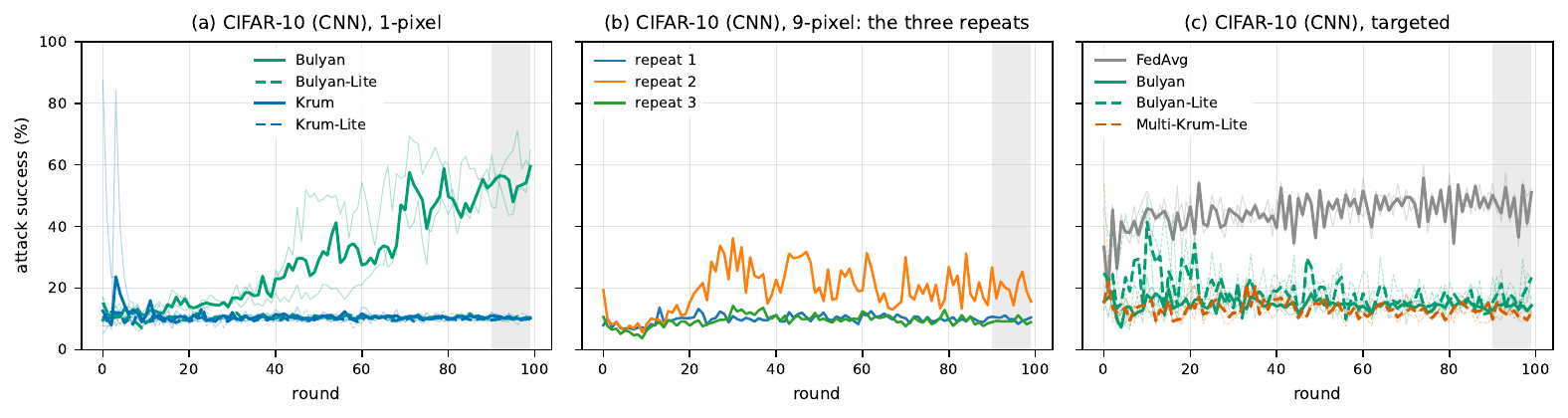}
    \caption{The three trajectories that \Cref{subsec:results} argues from,
    all on CIFAR-10 with the CNN. (a) Under the 1-pixel trigger Bulyan climbs
    from round forty onward while \emph{Bulyan-Lite} stays flat, the
    42.6-point gap of \Cref{subsec:results}. The two Krum lines are flat
    throughout. (b) \emph{Bulyan-Lite} under the 9-pixel trigger, one line per
    repeat, the cell \Cref{fig:repeat-grid} also draws. (c) The targeted
    attack, with undefended averaging, the Bulyan pair and
    \emph{Multi-Krum-Lite}, where \emph{Bulyan-Lite} ends above its parent.
    Solid lines are averaging and the parameter-distance rules, dashed the
    fingerprint variants. In (a) and (c) each line is the median across the
    three repeats, drawn over the repeats themselves, and (b) draws the repeats
    alone. The shaded band on the right is the final 10 rounds, which every
    benchmark table averages.}
    \label{fig:rounds-main}
\end{figure*}

\Cref{fig:rounds-pixel,fig:rounds-label} draw all 100 rounds of every
configuration, with the averaged final 10 rounds shaded. Rows are
architectures, and each attack contributes a column of clean accuracy and a
column of attack success. Solid lines are the parameter-distance rules and
dashed lines their fingerprint variants in the same hue, so each comparison
sits within one panel. Each line is the median over the three repeats, and the
band around a variant spans them.

\begin{figure*}[p]
    \centering
    \includegraphics[width=\textwidth]{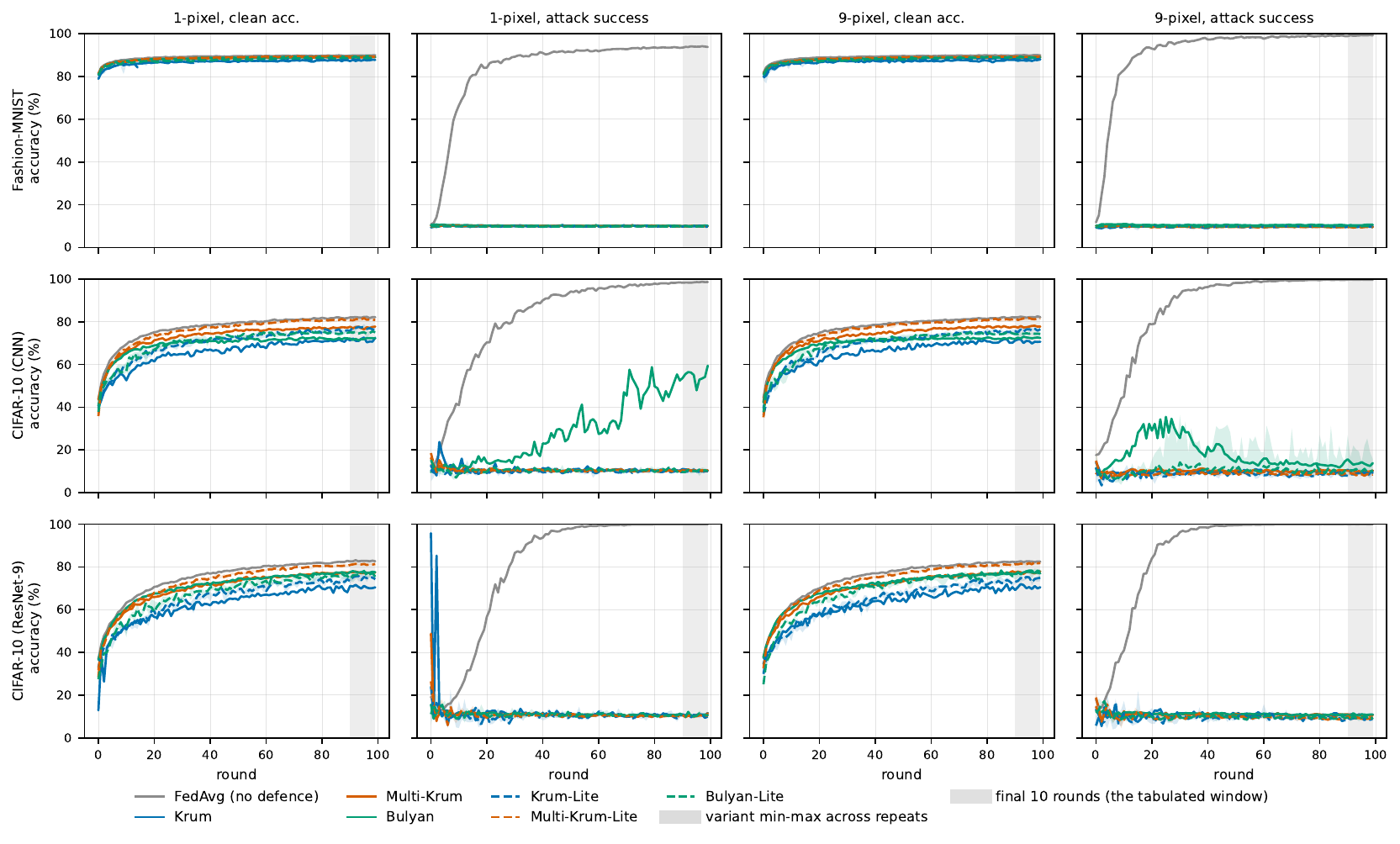}
    \caption{Per-round clean accuracy and attack success under the two pixel
    triggers. Federated averaging exceeds 93\% attack success in every panel,
    and every selection rule holds the attack near the 10\% chance floor
    except Bulyan in the CIFAR-10 CNN row. Under the 1-pixel trigger Bulyan
    climbs from round forty onward, to 52.79\% over the tabulated window, while
    \emph{Bulyan-Lite} stays at 10.18\%. Under the 9-pixel trigger Bulyan rises
    and falls back within the run, averaging 13.56\% over the window against
    10.04\% for \emph{Bulyan-Lite}, whose band there, the widest across-repeat
    spread in the grid, \Cref{fig:repeat-grid} draws per repeat.}
    \label{fig:rounds-pixel}
\end{figure*}

\begin{figure*}[p]
    \centering
    \includegraphics[width=\textwidth]{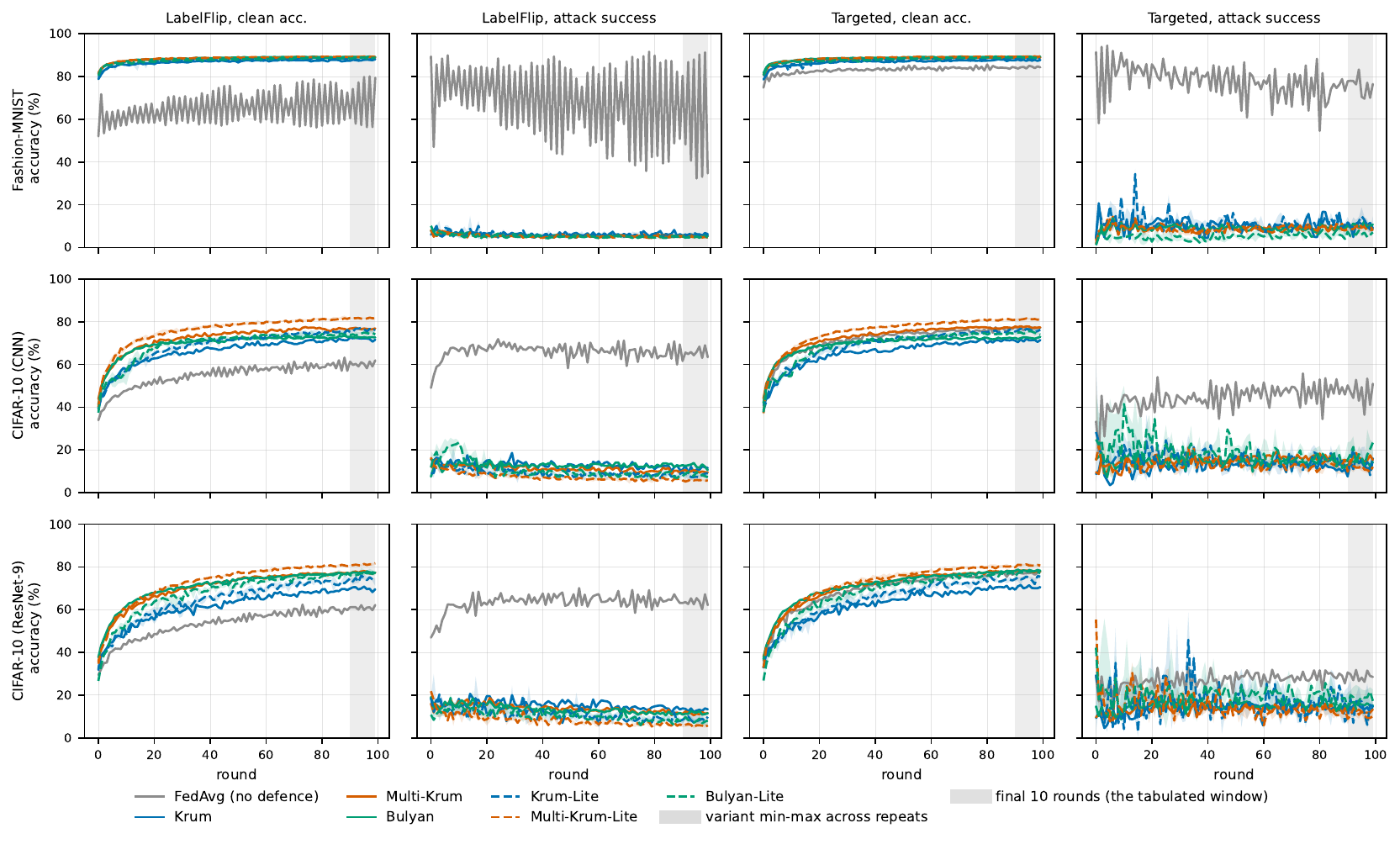}
    \caption{Per-round clean accuracy and attack success under the two label
    attacks, both weaker undefended than the pixel triggers. Federated
    averaging reaches 62 to 65\% under label flipping and 29 to 75\% under the
    targeted attack, and every selection rule holds both. The ResNet-9
    targeted panel is the configuration excluded as non-discriminative, where
    undefended averaging reaches only 29.1\%.}
    \label{fig:rounds-label}
\end{figure*}

\clearpage
\section{Ablation Details}
\label{sec:app-ablation}

The five measures of \Cref{subsec:ablation} are computed as follows. The
flattened client update counts floating-point tensors only. The Gaussian
projection is drawn once and shared by every client and round. The
penultimate features are averaged over the pseudo-patterns, and the softmax
measure applies a softmax to the fingerprint before comparing. The penultimate
width \(h\) is 1{,}568, 128 and 512 for the three architectures. AUC is
computed as the rank-sum statistic with ties credited one half. It scores the
whole ordering, whereas precision@$m$ looks only at the $m$ clients that would
be selected. \Cref{tab:ablation-precision-p3} reports precision@$m$,
\Cref{fig:ablation-measures} draws three of the five measures, and
\Cref{tab:ablation-spread-p3} lists the cells whose AUC varies most across
repeats, with their three values.

The softmax trails the fingerprint in every cell, at $0.551$ to $0.953$ AUC
(\Cref{tab:ablation-auc-p3}), and
the penultimate representation in five of six while tying the sixth, at
$0.739$ to $1.000$. The fingerprint and the parameter distance each vary by at
most $0.008$ AUC across their three repeats, while the random projection, the
penultimate representation and the softmax each exceed the $0.05$ flag in
three of six cells.

\AblationAUC

\AblationPrecision

\AblationSpread

\begin{figure}[t]
    \centering
    \includegraphics[width=\linewidth]{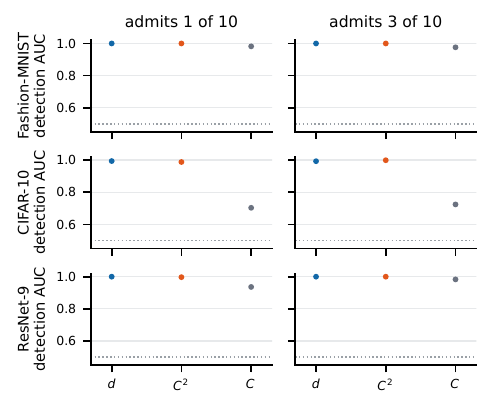}
    \caption{Detection AUC of three of the five measures, per architecture and
    per selection size, namely the full parameter vector
    ($d$), the logit fingerprint ($C^{2}$) and the random projection ($C$).
    Each point is the median over three repeats of the same ranking problem.
    The axis starts at $0.45$, below the dotted line at $0.5$, which is
    chance. The penultimate and softmax measures are in
    \Cref{tab:ablation-auc-p3}.}
    \label{fig:ablation-measures}
\end{figure}

\clearpage
\section{Selection Cost in Practice}
\label{sec:app-wallclock}

\begin{proof}[Proof of \Cref{lem:complexity}]
The \(\binom{n}{2} = O(n^{2})\) pairwise distances need computing only once,
each over the \(C^{2}\) values of two fingerprints, and are reused in every
pass. Each pass sorts every remaining client's distances, at \(O(n \log n)\)
per client and \(O(n^{2} \log n)\) in all, and there are \(m\) passes.
Replacing \(C^{2}\) by \(d\) gives the parameter-distance bound.
\end{proof}

The cost claim of this work is the count of elements that selection compares,
\(C^{2} = 100\) per pair of clients against \(d\), the quantity the
substitution changes directly. \Cref{tab:wallclock} reports the wall clock
alongside it. Over one round of 10 clients on an RTX A4000, selection on
fingerprints is 1.9, 5.2 and 8.2 times faster than Krum on the three
architectures, two to four orders of magnitude less than the element ratios.
The implementation computes each of the 45 pairwise distances in both
directions, 90 comparisons in all and twice what \Cref{lem:complexity}
requires, in both arms alike. At this scale a comparison is dominated by a
fixed per-call cost rather than by the elements it touches. Krum spends
22.0\,ns per element on the Fashion-MNIST CNN and selection on fingerprints
3{,}297\,ns, the same overhead spread over far fewer elements. The element
count does not depend on the implementation, which is why the claim is stated
in it. \Cref{tab:overhead} gives the per-round overhead that
\Cref{subsec:cost} summarises, the cost the client pays included.

\SelectionWallClock

\OverheadTable

\end{document}